\documentclass{article}

\usepackage[preprint,nonatbib]{neurips_2021}

\usepackage[utf8]{inputenc} 
\usepackage[T1]{fontenc}    
\usepackage{hyperref}       
\usepackage{url}            
\usepackage{booktabs}       
\usepackage{amsfonts}       
\usepackage{nicefrac}       
\usepackage{microtype}      
\usepackage{xcolor}         

\usepackage{fullpage,graphicx, setspace, latexsym, cite,amsmath,amssymb,color,subfigure}
\usepackage{amssymb} 
\usepackage{enumitem}
\usepackage{amsmath} 
\usepackage{amsthm}
\usepackage[]{algorithm}
\usepackage[]{algorithmic}

\newtheorem{theorem}{Theorem}

\newtheorem{lemma}{Lemma}
\newtheorem{defn}{Definition}
\newtheorem{ex}{Example}

\usepackage{float}
\newcommand{\targ}[3]{#1_{#2}^{(#3)}}

\def\R{\mathbb{R}}

\def\E{\mathbb{E}}
\def\D{\mathcal{D}}

\def\K{\mathcal{K}}

\def\O{\mathcal{O}}

\def\1{\mathbf{1}}

\DeclareMathOperator*{\argmax}{arg\,max}

\title{Improved Regret Bounds for Online Submodular Maximization}

\author{%
  Omid Sadeghi\\
  Department of Electrical and Computer Engineering\\
  University of Washington\\
  Seattle, WA 98195\\
  \texttt{omids@uw.edu}\\
   \And
   Prasanna Raut\\
   Department of Mechanical Engineering\\
   University of Washington\\
   Seattle, WA 98195\\
   \texttt{raut@uw.edu}\\
   \AND
   Maryam Fazel\\
  Department of Electrical and Computer Engineering\\
  University of Washington\\
  Seattle, WA 98195\\
  \texttt{mfazel@uw.edu}\\
}

\begin{document}

\maketitle

\begin{abstract}
  In this paper, we consider an online optimization problem over $T$ rounds where at each step $t\in[T]$, the algorithm chooses an action $x_t$ from the fixed convex and compact domain set $\mathcal{K}$. A utility function $f_t(\cdot)$ is then revealed and the algorithm receives the payoff $f_t(x_t)$. This problem has been previously studied under the assumption that the utilities are adversarially chosen monotone DR-submodular functions and $\mathcal{O}(\sqrt{T})$ regret bounds have been derived. We first characterize the class of strongly DR-submodular functions and then, we derive regret bounds for the following new online settings: $(1)$ $\{f_t\}_{t=1}^T$ are monotone strongly DR-submodular and chosen adversarially, $(2)$ $\{f_t\}_{t=1}^T$ are monotone submodular (while the average $\frac{1}{T}\sum_{t=1}^T f_t$ is strongly DR-submodular) and chosen by an adversary but they arrive in a uniformly random order, $(3)$ $\{f_t\}_{t=1}^T$ are drawn i.i.d. from some unknown distribution $f_t\sim \mathcal{D}$ where the expected function $f(\cdot)=\mathbb{E}_{f_t\sim\mathcal{D}}[f_t(\cdot)]$ is monotone DR-submodular. For $(1)$, we obtain the first logarithmic regret bounds. In terms of the second framework, we show that it is possible to obtain similar logarithmic bounds with high probability. Finally, for the i.i.d. model, we provide algorithms with $\tilde{\mathcal{O}}(\sqrt{T})$ stochastic regret bound, both in expectation and with high probability. Experimental results demonstrate that our algorithms outperform the previous techniques in the aforementioned three settings.
\end{abstract}

\section{Introduction}
Online Convex Optimization (OCO) is a well-studied framework for sequential prediction in face of uncertainty inherent in the arriving data \cite{10.5555/3041838.3041955,OPT-013,MAL-018,bubeck2011introduction}. OCO can be interpreted as a repeated game between a learner and an adversary in which at each round $t\in[T]$, the player chooses an action $x_t$ from a fixed convex and compact domain $\K$ and upon committing to this action, a convex loss function $f_t: \K \to \R$ is revealed and the learner incurs a loss of $f_t(x_t)$. In the non-stochastic adversary model, the sequence of loss functions $\{f_t\}_{t=1}^T$ is assumed to be selected by an adversary who knows the learner's algorithm (but not the potential randomness used for prediction). The goal is to choose $\{x_t\}_{t=1}^T$ such that the \emph{regret} of the learner defined as $R_T=\sum_{t=1}^T f_t(x_t)-\min_{x\in \K}\sum_{t=1}^Tf_t(x)$ is minimized. For this setting, several algorithms have been proposed to obtain the provably optimal $\O(\sqrt{T})$ regret bound \cite{orabona2019modern,OPT-013,MAL-018,bubeck2011introduction}. Similarly, for the case where $\{f_t\}_{t=1}^T$ is a sequence of monotone DR-submodular utility functions (defined later in Section \ref{dr}), \cite{NIPS2008_5751ec3e,golovin2014online,pmlr-v84-chen18f} obtained $\O(\sqrt{T})$ regret bounds against the $(1-\frac{1}{e})$-approximation to the best fixed decision in hindsight, i.e., regret being defined as $R_T=(1-\frac{1}{e})\max_{x\in \K}\sum_{t=1}^T f_t(x)-\sum_{t=1}^T f_t(x_t)$. If the sequence of loss functions is strongly convex, algorithms with logarithmic regret bounds have been introduced \cite{hazan2007logarithmic}. However, despite the fact that strongly DR-submodular functions were introduced by \cite{NIPS2008_5751ec3e,bian2020continuous}, a concrete characterization of such functions and similar logarithmic regret bounds for them is missing in the literature.\\
In some scenarios, despite the adversarial nature of the sequence of loss functions, the online input does not have a temporal structure and online data is streamed without particular order \cite{10.1007/978-3-642-15775-2_16,devanur2009adwords,doi:10.1137/1.9781611973730.93}. For instance, consider the problem of bidding in repeated auctions for online advertising in which at each round $t\in[T]$, an impression (e.g., online viewer) arrives and becomes available for sale through an auction. The advertiser needs to bid to win the auction and the corresponding impression. In this example, the viewers do not arrive in a particular order and therefore, their order of arrival could be assumed to be random.\\
Moreover, in many applications, the sequence of loss functions $\{f_t\}_{t=1}^T$ is not chosen arbitrarily (adversarially). For instance, in empirical risk minimization, the learner wants to minimize a loss function $f$ which is the expectation of empirical loss functions $f_t(\cdot)=f(\cdot;\omega_t)~\forall t\in[T]$, where $\omega_t$ is drawn i.i.d. from a fixed unknown distribution $\D$, i.e., $f(\cdot)=\E_{\omega \sim \D}[f(\cdot;\omega)]$. In this setting, we aim to minimize the \emph{stochastic regret} defined as $\text{SR}_T=\sum_{t=1}^T f(x_t)-T\min_{x\in \K}f(x)$ ($T(1-\frac{1}{e})\max_{x\in \K}f(x)-\sum_{t=1}^T f(x_t)$ in the DR-submodular setting).\\
In this paper, we fill the gaps in the online submodular maximization literature through studying online strongly DR-submodular maximization problems in the adversarial setting, and also online submodular problems against the aforementioned two stochastic adversary models.
\subsection{Preliminaries}\label{dr}
\textbf{Notations.} The set $\{1,2,\dots,T\}$ is denoted by $[T]$. For a vector $x\in \R^n$, we use $x_i$ to denote the $i$-th entry of $x$. The inner product of two vectors $x,y\in\mathbb{R}^n$ is denoted by either $\langle x, y \rangle$ or $x^T y$. Moreover, for two vectors $x,y\in \mathbb{R}^n$, we have $x\preceq y$ if $x_i \leq y_i$ holds for every $i\in[n]$. A function $f:\mathbb{R}^n \to \mathbb{R}$ is called monotone if for all $x,y$ such that $x\preceq y$, $f(x)\leq f(y)$ holds. The dual norm $\|\cdot\|_*$ of a norm $\|\cdot\|$ is defined as $\|y\|_*=\max_{x:\|x\|\leq 1}\langle y,x\rangle$. A differentiable function $f:\K \to \R$ is called $\beta$-Lipschitz with respect to $\|\cdot\|$ if for all $x,y\in \K$, we have $|f(y)-f(x)|\leq \beta \|y-x\|$, or equivalently, $\|\nabla f(x)\|_*\leq \beta$ holds. The diameter of a set $\K$ with respect to $\|\cdot\|$ is defined as $R=\max_{x,y\in \K}\|y-x\|$.\\
\textbf{Submodular functions \cite{pmlr-v54-bian17a}.} A continuous differentiable function $f:\K \rightarrow \mathbb{R}$, $\K\subset \mathbb{R}_+^n$, is called submodular if for all $x,y\in \K, x\succeq y$ and $i\in[n]$ such that $x_i=y_i$, we have $\nabla_i f(x) \leq \nabla_i f(y)$. For a twice differentiable function $f$, it is submodular if and only if the off-diagonal entries of its Hessian matrix $\nabla^2 f$ are non-positive.\\
\textbf{DR-submodular functions \cite{pmlr-v54-bian17a,sadeghi2020online}.} A differentiable function $f:\K \rightarrow \mathbb{R}$, $\K\subset \mathbb{R}_+^n$, is called DR-submodular if its gradient is an order-reversing mapping, i.e., for all $x,y$ such that $x\succeq y$, $\nabla f(x) \preceq \nabla f(y)$ holds. For a twice differentiable function $f$, it is DR-submodular if and only if its Hessian matrix $\nabla^2 f$ is entry-wise non-positive. While DR-submodularity and concavity are equivalent for the special case of $n=1$, DR-submodular functions are generally non-concave. Nonetheless, an important consequence of DR-submodularity is concavity along non-negative directions \cite{pmlr-v54-bian17a,doi:10.1137/080733991}, i.e., for all $x,y$ such that $x\preceq y$, we have $f(y)\leq f(x)+\langle \nabla f(x),y-x\rangle$. It is easy to see that the class of DR-submodular functions are a subset of submodular functions and in fact, submodularity along with coordinate-wise concavity is equivalent to DR-submodularity \cite{pmlr-v54-bian17a}.\\
There are many functions which satisfy the DR-submodularity property including:\\
$\bullet$ \textbf{Indefinite quadratic functions.} Let $f(x)=\frac{1}{2}x^T Ax+a^T x+c$ where $A$ is a symmetric matrix. If $A$ is entry-wise non-positive, $f$ is a DR-submodular function. Such quadratic utility functions have a wide range of applications. In particular, price optimization with continuous prices \cite{NIPS2016_6301} and computing stability number of graphs \cite{motzkin_straus_1965} are both non-concave DR-submodular quadratic optimization problems.\\
$\bullet$ \textbf{Concave functions with negative dependence.} Let $d\geq 2$. If $h_i:\R_+ \to \R$ is concave for all $i\in[n]$ and $\theta_{i_1,\dots,i_r}\leq0$ for all $r\in[d]$ and $(i_1,\dots,i_r)\subseteq [n]$, the following function $f:\R_+^n \to \R$ is DR-submodular:
\begin{equation*}
f(x)=\sum_{i=1}^n h_i (x_i)+\sum_{(i,j):i\neq j}\theta_{ij}x_i x_j+\dots+\sum_{(i_1,\dots,i_d):i_r\neq i_s~\forall r,s\in[d]}\theta_{i_1,\dots,i_d}x_{i_1}\dots x_{i_d}.
\end{equation*}
See \cite{bian2020continuous,sadeghi2020single} for more examples and applications of continuous submodular functions.\\
For a DR-submodular function $f$, $f$ is $L$-smooth over non-negative directions with respect to $\|\cdot\|$ if for all $x,y$ such that $x\preceq y$, we have $f(y)-f(x)\geq \langle \nabla f(x), y-x\rangle -\frac{L}{2}\|y-x\|^2$.
\begin{ex}
    Let $f(x)=\frac{1}{2}x^TAx+a^Tx+c$ where $A$ is symmetric and entry-wise non-positive. If for all $i,j\in[n]$, $A_{ij}\geq -L$ holds, the function $f$ is $L$-smooth with respect to $\|\cdot\|_1$.
\end{ex}
\subsection{Related work}
\textbf{Online optimization in the i.i.d. model.} For the setting where the objective functions $\{f_t\}_{t=1}^T$ are drawn i.i.d. from an unknown distribution $\D$ with a convex mean $f(\cdot)=\E_{\D}[f_t(\cdot)]$, several algorithms have been proposed that could be divided into two categories of projection-based and projection-free algorithms where the latter is most relevant to our work. In particular, \cite{10.5555/3042573.3042808} proposed the Online Frank-Wolfe (OFW) algorithm which achieves a nearly optimal $\tilde{\O}(\sqrt{T})$ stochastic regret bound with high probability. The OFW algorithm requires to access exact gradient of $\{f_t\}_{t=1}^T$ and has an average $\O(T)$ per-iteration computational cost. To remedy this issue, more recently, \cite{pmlr-v80-chen18c} proposed the One-Shot Frank-Wolfe (OSFW) algorithm with $\O(T^{2/3})$ stochastic regret bound in expectation. Note that the OSFW algorithm obtains similar bounds for the setting where $f$ is monotone continuous DR-submodular. The OSFW algorithm only uses unbiased stochastic gradient estimates of loss functions and has an $\O(1)$ cost per iteration. However, the derived $\O(T^{2/3})$ stochastic regret bound is sub-optimal. To bridge this gap, \cite{xie2020efficient} introduced the Online stochastic Recursive Gradient-based Frank-Wolfe (ORGFW) algorithm which not only has a nearly optimal $\tilde{O}(\sqrt{T})$ stochastic regret bound with high probability, but it also maintains a low $\O(1)$ computational cost per round.\\
\textbf{Stochastic optimization.} Online optimization in the i.i.d. model is closely related but different from the stochastic optimization problem \cite{birge1997springer}. In online optimization in the i.i.d. model, the goal is to choose a \emph{sequence} $\{x_t\}_{t=1}^T$ of decision variables that has low stochastic regret and the algorithm requires to update the decision variables as soon as new data arrives online \cite{pmlr-v80-chen18c}. In contrast, stochastic optimization focuses on the quality of the \emph{final} output of the algorithm \cite{pmlr-v84-mokhtari18a} and aims to find an approximate optimal point of the underlying objective function, where the performance is measured by the convergence rate.\\
\textbf{Online convex optimization in the random order model.} \cite{icml2020_4841} studied OCO in a weaker model where the loss functions are still chosen adversarially but their order of arrival is random. This model is termed Random Order Online Convex Optimization (ROOCO). ROOCO is a natural middle ground between the standard adversarial OCO setting and the model where the loss functions are i.i.d. drawn from some unknown underlying distribution. In \cite{icml2020_4841}, the authors assumed that the average of the sequence of loss functions is $\alpha$-strongly convex while each individual loss function may not even be convex. They showed that if all the loss functions are quadratic, with probability at least $1-\delta$, Online Gradient Descent (OGD) with step size $\eta_t=\frac{1}{\alpha t}~\forall t\in[T]$ obtains a regret bound of $\O(\frac{\beta^2}{\alpha^3} {\rm log}^2 T)$, where $\beta$ is the Lipschitz constant of the loss functions. Moreover, for general loss functions, OGD with the same choice of step sizes obtains a regret bound of $\O(\frac{n\beta^2}{\alpha^3} {\rm log}^2 T)$ ($n$ is the dimension of the domain space) with probability at least $1-\delta$. They also provided two hardness results for the standard adversarial OCO model: First, they showed that there exists a sequence of (not necessarily convex) quadratic loss functions where the cumulative loss is strongly convex for which OGD with arbitrary step sizes suffers a linear regret with a non-zero probability. Also, they proved that there exists a sequence of convex loss functions with strongly convex cumulative loss where OGD with arbitrary step sizes suffers an $\Omega (\sqrt{T})$ regret. These two results verify that the random order assumption is crucial for obtaining logarithmic regret bounds.\\
\textbf{Offline DR-submodular maximization.} \cite{pmlr-v54-bian17a} proposed a variant of the Frank-Wolfe algorithm for maximizing a monotone DR-submodular function $f$ subject to a convex domain $\K$. Starting from $x^{(1)}=0$, the algorithm performs $K$ Frank-Wolfe updates where at each iteration $k\in[K]$, it finds $v_k$ such that $v_k=\argmax_{x\in \K}\langle x,\nabla f(x^{(k)})\rangle$, and performs the update $x^{(k+1)}=x^{(k)}+\frac{1}{K}v_k$. \cite{pmlr-v54-bian17a} showed that the output of this algorithm ($x^{(K+1)}$) obtains the provably optimal approximation ratio of $1-\frac{1}{e}$.
\begin{table}[t]\centering
	\begin{tabular}{c|c|c|c|c}
		Paper & Setting & Approx. Ratio & Regret & Guarantee\\ \hline
		\cite{pmlr-v84-chen18f} & adversarial, DR-submodular & $1-1/e$ & $\O(\sqrt{T})$ & deterministic\\ \hline
		Algorithm \ref{alg1} & adversarial, strongly DR-submodular & $1-1/e$ & $\O(\ln T)$ & deterministic\\ \hline
		Algorithm \ref{alg1} & random order, submodular$^{(a)}$ & $1-1/e$ & $\O(\ln^2 T)$ & w.h.p.\\ \hline
		\cite{pmlr-v80-chen18c} & i.i.d., DR-submodular & $1/e$ & $\O(T^{2/3})$ & in expectation\\ \hline
		Algorithm \ref{alg:exp4} & i.i.d., DR-submodular & $1-1/e$ & $\O(\sqrt{T})$ & in expectation, w.h.p.\\ \hline
		Algorithm \ref{alg:exp5} & i.i.d., DR-submodular & $1/e$ & $\O(\sqrt{T})$ & in expectation, w.h.p.\\ \hline
	\end{tabular}
	\caption{Comparison of online submodular algorithms. Note that in (a), we also assume that $\frac{1}{T}\sum_{t=1}^T f_t$ is strongly DR-submodular.}
	\label{table1}
\end{table}
\subsection{Contributions}
In this paper, we study online submodular maximization in three different online settings. Our contributions are listed as follows:\\
$\bullet$ In Section \ref{sdra}, we first introduce and characterize the class of strongly DR-submodular functions. We then propose Algorithm \ref{alg1} which achieves logarithmic regret bound in the adversarial setting with strongly DR-submodular utility functions $\{f_t\}_{t=1}^T$.\\
$\bullet$ We consider the stochastic random order model in Section \ref{ro}, where the utility functions are submodular and still chosen adversarially but they arrive in a uniformly random order. For this setting, we leverage concentration inequalities for sampling without replacement as our main non-standard tool to show that if the average of utility functions $\frac{1}{T}\sum_{t=1}^Tf_t$ is strongly DR-submodular, we can exploit Algorithm \ref{alg1} to achieve logarithmic regret bound even if each individual $f_t$ is only submodular.\\
$\bullet$ In Section \ref{iid}, we study the stochastic i.i.d. model where the sequence of generally non DR-submodular utility functions $\{f_t\}_{t=1}^T$ are drawn from an unknown distribution $\D$ with DR-submodular mean $f$, i.e., $\E_{\D}[f_t(\cdot)]=f(\cdot)$, and we only have access to unbiased stochastic gradient estimates of $\{f_t\}_{t=1}^T$. We propose Algorithm \ref{alg:exp4} and Algorithm \ref{alg:exp5} for this setting, and we obtain nearly optimal $\tilde{\O}(\sqrt{T})$ stochastic regret bounds, both in expectation and with high probability for each algorithm. Algorithm \ref{alg:exp4} manages to obtain the optimal approximation ratio $1-\frac{1}{e}$, however it requires $\O(T^{5/2})$ overall gradient evaluations. On the other hand, Algorithm \ref{alg:exp5} only achieves a $\frac{1}{e}$ approximation ratio while requiring a much lower $\O(T)$ total gradient evaluations.\\
Finally, in Section \ref{exp}, we test our online algorithms through numerical experiments on synthetic and real world datasets and highlight their superiority compared to previous techniques for the aforementioned three online frameworks. A summary of the results of this paper in comparison with prior works is presented in Table \ref{table1}. All missing proofs are provided in the Appendix.
\section{Online strongly DR-submodular maximization in the adversarial setting}\label{sdra}
We first define strongly DR-submodular functions below.
\begin{defn}(\textbf{Strong DR-submodularity \cite{bian2020continuous,NIPS2017_58238e9a}}) For $\mu \geq 0$, a function $f:\mathcal{K}\to \R$, $\mathcal{K}\subset \R_+^n$, is called $\mu$-strongly DR-submodular with respect to $\|\cdot\|$ if for all $x\in \mathcal{K}$ and $v\succeq 0$ (or $v\preceq 0$), the following holds:
\begin{equation*}
    f(x+v)\leq f(x)+\langle \nabla f(x),v\rangle -\frac{\mu}{2}\|v\|^2.
\end{equation*}
\end{defn}
Note that if $\mu =0$, $f$ is DR-submodular.
The following Lemma characterizes the class of strongly DR-submodular functions with respect to $\|\cdot\|_2$.
\begin{lemma}\label{lm1}
A twice differentiable function $f:\K \rightarrow \mathbb{R}$, $\K\subset \mathbb{R}_+^n$, is $\mu$-strongly DR-submodular with respect to $\|\cdot\|_2$ if for all $x\in \K$, we have $\nabla^2_{ii}f(x)\leq -\mu~\forall i\in[n]$ and $\nabla^2_{ij}f(x)\leq 0~\forall i\neq j$.
\end{lemma}
\begin{proof}
For any $z,v\in \K$ and $i\in[n]$, we have:
\begin{align*}
    [\nabla^2 f(z)v]_i+\mu v_i = \nabla^2_{ii}f(z)v_i+\sum_{j\neq i}\nabla^2_{ij}f(z)v_j+\mu v_i=\underbrace{(\nabla^2_{ii}f(z)+\mu)}_{\leq 0}\underbrace{v_i}_{\geq 0}+\sum_{j\neq i}\underbrace{\nabla^2_{ij}f(z)}_{\leq 0}\underbrace{v_j}_{\geq 0}\leq 0.
\end{align*}
Therefore, $\nabla^2f(z)v+\mu v\preceq 0$ holds for all $z,v\in \K$. We can use the mean value theorem twice to write:
\begin{align*}
    f(x+v)-f(x)-\langle\nabla f(x),v\rangle&= \int_{0}^1\langle \nabla f(x+tv),v\rangle dt - \langle \nabla f(x),v\rangle\\
    &=\int_{0}^1\langle \nabla f(x+tv)-\nabla f(x),v \rangle dt\\
    &=\int_{0}^1 \langle t\nabla^2f(z)v,v\rangle dt,
\end{align*}
where $z$ is in the line segment between $x$ and $x+tv$. Combining the above two inequalities, we have:
\begin{align*}
    f(x+v)-f(x)-\langle\nabla f(x),v\rangle=\int_{0}^1 \langle t\underbrace{(\nabla^2f(z)v+\mu v)}_{\preceq 0},\underbrace{v}_{\succeq 0}\rangle dt-\mu\int_{0}^1 t\langle v,v\rangle dt\leq\frac{-\mu}{2}\|v\|_2^2.
\end{align*}
Thus, $f$ is $\mu$-strongly DR-submodular with respect to $\|\cdot\|_2$.
\end{proof}
Lemma \ref{lm1} states that for a twice differentiable DR-submodular function, if $\nabla^2_{ii}f(x)\leq -\mu$ holds for all $i\in[n]$ and $x\in\K$, the function is $\mu$-strongly DR-submodular with respect to $\|\cdot\|_2$. For instance, for the class of concave functions with negative dependence (introduced in section \ref{dr}), if $h_i$ is $\mu$-strongly concave for all $i\in[n]$, the function is $\mu$-strongly DR-submodular.\\
Consider the following adversarial online optimization problem over $T$ rounds: At each step $t\in[T]$, the algorithm chooses a decision variable $x_t\in \K$, where $\K$ is the fixed convex and compact domain with diameter $R$ with respect to $\|\cdot\|$ and $0\in \K$. Once the algorithm commits to the action $x_t$, the adversary reveals a monotone utility function $f_t$ which is $\mu$-strongly DR-submodular and $L$-smooth with respect to $\|\cdot\|$. Without loss of generality and for ease of notation, we assume the utility functions are normalized, i.e., $f_t(0)=0~\forall t\in[T]$. The overall goal is to maximize the cumulative utility $\sum_{t=1}^T f_t(x_t)$ or equivalently, minimize the $(1-\frac{1}{e})$-regret $R_T$ defined as:
\begin{equation*}
    R_T=(1-\frac{1}{e})\max_{x\in \K}\sum_{t=1}^T f_t(x)-\sum_{t=1}^T f_t(x_t),
\end{equation*}
where $1-\frac{1}{e}$ is the optimal polynomial time approximation ratio for offline monotone DR-submodular maximization. In other words, the regret compares the decisions of the algorithm with the $(1-\frac{1}{e})$-approximation to the optimal solution in hindsight.\\
For the setting with $\mu=0$, i.e., when all utility functions are only DR-submodular, \cite{pmlr-v84-chen18f} proposed the Meta-Frank-Wolfe algorithm with a provably optimal $\O(\sqrt{T})$ regret. Since for all $t\in[T]$, $f_t$ remains unknown until the algorithm chooses the action $x_t$, the Meta-Frank-Wolfe algorithm runs $K=\O(\sqrt{T})$ instances of a no-regret online linear maximization algorithm (such as Follow the Regularized Leader (FTRL)) to mimic the Frank-Wolfe variant of \cite{pmlr-v54-bian17a} for offline DR-submodular maximization in the online setting, and outputs the average of the decisions of these instances at each round.
\begin{algorithm}[t]
	\caption{Online strongly DR-submodular maximization algorithm}
	\begin{algorithmic}\label{alg1}
		\STATE \textbf{Input}: $K>0$ is the number of inner loops, $T$ is the horizon, and $\mu>0$.
		\STATE \textbf{Output}: $\{x_t :1\leq t\leq T\}$.
		\STATE Choose an off-the-shelf online strongly concave maximization algorithm and initialize $K$ instances $\{\mathcal{E}_k\}_{k=1}^K$ of it for online maximization of $\mu$-strongly concave utility functions over $\K$.
		\FOR{$t=1$ {\bfseries to} $T$}
		\STATE Set $\targ{x}{t}{1}=0$.
		\FOR{$k=1$ {\bfseries to} $K$}
		\STATE Let $\targ{v}{t}{k}$ be the vector selected by $\mathcal{E}_k$.
		\STATE $\targ{x}{t}{k+1}=\targ{x}{t}{k}+\frac{1}{K}\targ{v}{t}{k}$.
		\ENDFOR
		\STATE Play $x_t=\targ{x}{t}{K+1}$, observe the function $f_t$ and the reward $f_t (x_t)$.
		\STATE Feedback $\langle \targ{v}{t}{k},\nabla f_t (\targ{x}{t}{k})\rangle -\frac{\mu}{2}\|v_t^{(k)}\|^2$ as the payoff to be received by $\mathcal{E}_k$.
		\ENDFOR
	\end{algorithmic}
\end{algorithm}
We now propose a modified version of the Meta-Frank-Wolfe algorithm of \cite{pmlr-v84-chen18f} for strongly DR-submodular utility functions to be able to obtain improved logarithmic regret bound in this case. The algorithm is presented in Algorithm \ref{alg1}. The algorithm runs $K$ instances $\{\mathcal{E}_k\}_{k=1}^K$ of no-regret online strongly concave maximization algorithms (such as Follow the Leader (FTL)) where at each round $t\in[T]$, the instance $\mathcal{E}_k$ chooses the action $\targ{v}{t}{k}$ and upon committing to this decision, it receives a $\mu$-strongly concave payoff of $\langle \targ{v}{t}{k},\nabla f_t (\targ{x}{t}{k})\rangle -\frac{\mu}{2}\|v_t^{(k)}\|^2$. Algorithm \ref{alg1} outputs $x_t=\frac{1}{K}\sum_{k=1}^K\targ{v}{t}{k}$. Note that the FTL update rule for $\targ{v}{t}{k}$ is computationally efficient in many cases. To be precise, in case of the norm $\|\cdot\|_2$, we can write:
\begin{equation*}
    \targ{v}{t}{k}=\argmax_{x\in \K} \langle x,\sum_{s=1}^{t-1}\nabla f_s(\targ{x}{s}{k})\rangle -\frac{\mu(t-1)}{2}\|x\|_2^2=\text{Proj}_{\K}\big(\frac{1}{\mu (t-1)}\sum_{s=1}^{t-1}\nabla f_s(\targ{x}{s}{k})\big),
\end{equation*}
where $\text{Proj}_{\K}$ denotes Euclidean projection onto set $\K$.\\
The regret guarantee of Algorithm \ref{alg1} is provided below.
\begin{theorem}\label{thm1}
    Assume that $\{f_t\}_{t=1}^T$ is normalized (i.e., $f_t(0)=0$), monotone, $\mu$-strongly DR-submodular, $L$-smooth and $\beta$-Lipschitz with respect to $\|\cdot\|$. Using Algorithm \ref{alg1} with $K=\O(\frac{T}{\ln T})$, we have:
    \begin{equation*}
        R_T=(1-\frac{1}{e})\sum_{t=1}^Tf_t(x^*)-\sum_{t=1}^T f_t(x_t)\leq \O(\ln T),
    \end{equation*}
    where $x^*=\argmax_{x\in \K}\sum_{t=1}^Tf_t(x)$.
\end{theorem}
Algorithm \ref{alg1} requires $\tilde{\O}(T^2)$ overall FTL updates to obtain the logarithmic regret bound derived in Theorem \ref{thm1}. As it was mentioned earlier, the FTL updates are easy to compute in many cases. However, the projection step onto $\K$ in the update rule might be computationally expensive for complex domains. In such cases, we can apply Algorithm \ref{alg1} to $\{\frac{1}{W}\sum_{t=\tau W+1}^{(\tau+1)W}f_t\}_{\tau=0}^{\lceil \frac{T}{W}\rceil -1}$ instead. If we choose $W=\O(T^{\epsilon})$ and $K=\O(T^{1-\epsilon})$ for some $\epsilon \in [0,1)$, Algorithm \ref{alg1} obtains an $\O(T^{\epsilon}\ln T)$ regret bound while performing $\O(T^{2-2\epsilon})$ overall FTL updates. For instance, if we set $\epsilon=\frac{1}{4}$, Algorithm \ref{alg1} achieves an $\tilde{\O}(T^{1/4})$ regret bound which is a significant improvement over the $\O(\sqrt{T})$ regret bound of the Meta-Frank-Wolfe Algorithm of \cite{pmlr-v84-chen18f} with the same $\O(T^{3/2})$ number of projections.
\section{Online submodular maximization in the random order model}\label{ro}
In this section, we focus on the random order adversary model where the sequence of objective functions $\{f_t\}_{t=1}^T$ is still chosen adversarially, but they arrive in a uniformly random order. Specifically, we first consider an online submodular maximization problem in which the submodular utility function at step $t\in[T]$ is quadratic, $f_t(x)=\frac{1}{2}x^T A^{(t)} x+(a^{(t)})^Tx$ where $A^{(t)}$ is symmetric, every off-diagonal entry $\targ{A}{ij}{t}~\forall i\neq j\in[n]$ is in the range $[-L,0]$, and every diagonal entry of the average matrix $\frac{1}{T}\sum_{t=1}^T A^{(t)}$ is in the range $[-L,-\mu]$ while each individual $|\targ{A}{ii}{t}|\leq L$ for all $i\in[n]$. In other words, each utility function $\{f_t\}_{t=1}^T$ is submodular, and the average of utility functions $\frac{1}{T}\sum_{t=1}^T f_t$ is $\mu$-strongly DR-submodular with respect to $\|\cdot\|_2$ and $L$-smooth with respect to $\|\cdot\|_1$. We first provide the following concentration inequality for randomly permuted sums.
\begin{theorem}\label{thm2}(Theorem 4.3 of \cite{bercu2015concentration}) Let $\{c_{r,s}\}_{r,s\in[T]}$ be an array of real numbers from the range $[-m_c,m_c]$. Let $Z_T=\sum_{r=1}^T c_{r,\prod_T(r)}$ where $\prod_T$ is drawn from the uniform distribution over the set of all permutations of $\{1,\dots,T\}$. Then, for any $\lambda>0$, we have:
\begin{equation*}
    \mathbb{P}\big(|Z_T-\E[Z_T]|\geq \lambda\big)\leq 4{\rm exp}\big(-\frac{\lambda^2}{16(\frac{\theta}{T}\sum_{r,s=1}^T c_{r,s}^2+\frac{m_c \lambda}{3})}\big),
\end{equation*}
where $\theta=\frac{5}{2}\ln 3-\frac{2}{3}$.
\end{theorem}
Now, we apply the result of Theorem \ref{thm2} to our problem. Let $W\in [1,T]$. For a fixed $i\in[n]$, set:
\[c_{r,s} = 
\begin{cases}
	\targ{A}{ii}{s}-\frac{1}{T}\sum_{t=1}^T \targ{A}{ii}{t} &\quad\text{if } 1\leq r\leq W\\
	0 &\quad\text{o.w.}
\end{cases}\]
We can write:
\begin{align*}
    Z_T&=\sum_{r=1}^T c_{r,\Pi_T(r)}=\sum_{r=1}^W \targ{A}{ii}{\Pi_T(r)}-\frac{W}{T}\sum_{t=1}^T \targ{A}{ii}{t}.
\end{align*}
Taking expectation of the above equality, we obtain:
\begin{equation*}
    \E[Z_T]=\sum_{r=1}^W \E[\targ{A}{ii}{\Pi_T(r)}]-\frac{W}{T}\sum_{t=1}^T \targ{A}{ii}{t}=\sum_{r=1}^W (\frac{1}{T}\sum_{t=1}^T \targ{A}{ii}{t})-\frac{W}{T}\sum_{t=1}^T \targ{A}{ii}{t}=0.
\end{equation*}
Since $m_c = 2L$ in our setting, we have $\frac{\theta}{T}\sum_{r,s=1}^T c_{r,s}^2=\frac{\theta}{T}\sum_{r=1}^W\sum_{s=1}^T \big(\targ{A}{ii}{s}-\frac{1}{T}\sum_{t=1}^T \targ{A}{ii}{t}\big)^2\leq 4\theta L^2 W$. Plugging in $\lambda=W\epsilon$, we obtain:
\begin{align*}
    \mathbb{P}\big(|\frac{1}{W}\sum_{r=1}^W \targ{A}{ii}{\Pi_T(r)}-\frac{1}{T}\sum_{t=1}^T \targ{A}{ii}{t}|\geq \epsilon\big)&=\mathbb{P}\big(|\sum_{r=1}^W \targ{A}{ii}{\Pi_T(r)}-\frac{W}{T}\sum_{t=1}^T \targ{A}{ii}{t}|\geq W\epsilon\big)\\
    &\leq 4{\rm exp}\big(-\frac{W^2\epsilon^2}{16(4\theta L^2 W+\frac{2LW\epsilon}{3})}\big).
\end{align*}
$\mathbb{P}\big(|\frac{1}{W}\sum_{r=1}^W \targ{A}{ii}{\Pi_T(r)}-\frac{1}{T}\sum_{t=1}^T \targ{A}{ii}{t}|\geq \epsilon\big)$ is then bounded from above by
\[
\leq\begin{cases}
	4{\rm exp}\big(-\frac{W\epsilon^2}{128\theta L^2}\big) &\epsilon \leq 6\theta L\\
	4{\rm exp}\big(-\frac{3W\epsilon}{64L}\big) &\quad\text{o.w.}
\end{cases}.\]
If we choose $\epsilon \leq \frac{\mu}{2}$, the first case happens. Therefore, if we set $\frac{\delta}{nT}=4{\rm exp}\big(-\frac{W\epsilon^2}{128\theta L^2}\big)$, we can use the union bound to conclude that for every $i\in[n]$ and $\tau \in \{0,\dots,\lceil\frac{T}{W}\rceil-1\}$, with probability at least $1-\delta$, it holds that $\frac{1}{W}\sum_{t=\tau W+1}^{(\tau+1)W}\targ{A}{ii}{t}\leq -(\mu-\epsilon)\leq -\frac{\mu}{2}$ for any $W\geq W_0$ where $W_0=\frac{128\theta L^2}{\epsilon^2}\ln(\frac{4nT}{\delta})$. In other words, for large enough $W$, the average of each of the consecutive blocks of size $W$ of the utility functions is $(\frac{\mu}{2})$-strongly DR-submodular.\\
Now, consider general monotone submodular utility functions $\{f_t\}_{t=1}^T$. In this case, the Hessian of the utility functions are not fixed and therefore, we have to ensure the strong DR-submodularity of every block of size $W$ at every point $x\in \K$. The analysis in this setting is provided in the Appendix.\\
Considering that the strong DR-submodularity property propagates to the average of utility functions over sufficiently large blocks, we can apply Algorithm \ref{alg1} to these blocks to obtain logarithmic regret bounds. We summarize this result in the following theorem.
\begin{theorem}
    Assume that for all $t\in[T]$, the utility function $f_t$ is normalized, monotone and submodular, and the average of utility functions $\frac{1}{T}\sum_{t=1}^Tf_t$ is $\mu$-strongly DR-submodular and $L$-smooth. Then, with probability at least $1-\delta$, if we apply Algorithm \ref{alg1} to the utility functions $\{\frac{1}{W}\sum_{t=\tau W+1}^{(\tau+1)W}f_t\}_{\tau=0}^{\lceil \frac{T}{W}\rceil -1}$ for $W\geq W_0=\O(\frac{L^2}{\mu^2}\ln (\frac{nT}{\delta}))$, we have:
    \begin{equation*}
        \frac{1}{W}\sum_{\tau=0}^{\lceil \frac{T}{W}\rceil -1}\sum_{t=\tau W+1}^{(\tau+1)W}\big((1-\frac{1}{e})f_t(x^*)-f_t(z_{\tau})\big)\leq \O(\ln \frac{T}{W}).
    \end{equation*}
    Therefore, if we set $x_t=z_{\tau}$ for all $\tau \in \{0,\dots,\lceil \frac{T}{W}\rceil -1\}$ and $t\in \{\tau W+1,\dots,(\tau+1)W\}$, the following holds with probability at least $1-\delta$:
    \begin{equation*}
        R_T\leq \O(W\ln \frac{T}{W})=\O(\ln (\frac{nT}{\delta})\ln T).
    \end{equation*}
\end{theorem}
\section{Online DR-submodular maximization in the i.i.d. model}\label{iid}
In this section, we focus on the setting where the sequence of utility functions $\{f_t\}_{t=1}^T$ is drawn i.i.d. from some fixed unknown distribution $f_t\sim \mathcal{D}$ where $\E_{\D}[f_t(\cdot)]=f(\cdot)$. In this framework, the performance is measured via the notion of stochastic regret defined as $\alpha\text{-SR}_T=T\alpha f(x^*)-\sum_{t=1}^T f(x_t)$, where $x^*=\argmax_{x \in \K} f(x)$. Note that unlike the adversarial regret, the stochastic regret is defined with respect to the expected function $f$.\\
\textbf{Assumption 1. }We make the following assumptions on the utility functions:\\
$\bullet$ For all $t\in[T]$ and $x\in \K$, we only have access to the unbiased gradient oracle $\tilde{\nabla}f_t(x)$, i.e., $\E[\tilde{\nabla}f_t(x)]=\nabla f(x)$.\\
$\bullet$ There exists $\sigma>0$ such that for any $x\in \K$ and $t\in[T]$, $\|\tilde{\nabla} f_t(x)-\nabla f(x)\|_2\leq \sigma$ holds.\\
$\bullet$ $f$ is monotone DR-submodular and $L$-smooth.\\
$\bullet$ $\{f_t\}_{t=1}^T$ is $L$-smooth.\\
Note that we do not require the utility functions $\{f_t\}_{t=1}^T$ to be DR-submodular which further complicates the analysis in this setting.\\
If $f$ was known in advance, we would use the Frank-Wolfe variant of \cite{pmlr-v54-bian17a} for offline DR-submodular maximization. To be precise, for all $t+1\in [T]$, starting from $\targ{x}{t+1}{1}=0$, we would perform $K_t$ Frank-Wolfe updates where at each iteration $k\in[K_t]$, we choose $\targ{v}{t}{k}$ according to $\targ{v}{t}{k}=\argmax_{x\in \K}\langle x,\nabla f(x_{t+1}^{(k)})\rangle$, and perform the update $x_{t+1}^{(k+1)}=x_{t+1}^{(k)}+\frac{1}{K_t}\targ{v}{t}{k}$. However, $f$ is not available in advance. To tackle this problem, at round $t+1\in[T]$, we can use the average of utility functions $\{f_s\}_{s=1}^t$ observed so far as an estimate of $f$. Using this technique, we propose Algorithm \ref{alg:exp4} and provide its stochastic regret guarantee in the theorem below.
\begin{algorithm}[tbhp]
	\caption{}
	\begin{algorithmic}\label{alg:exp4}
		\STATE \textbf{Input:} $\{K_t=\sqrt{t}\}_{t=1}^T$ and $T>0$ is the horizon.
		\FOR{$t=1,\dots,T$}
		\STATE Play $x_t=x_t^{(K_{t-1}+1)}$ and observe $f_t$.
	    \STATE Set $x_{t+1}^{(1)}=0$.
		\FOR{$k=1,2,\dots,K_t$}
		\STATE $d_t^{(k)}=\frac{1}{t}\sum_{\tau=1}^t\tilde{\nabla} f_{\tau}(x_{t+1}^{(k)})$.
		\STATE $v_t^{(k)}=\argmax_{x\in \K}\langle x, d_t^{(k)} \rangle$.
		\STATE Set $x_{t+1}^{(k+1)}=x_{t+1}^{(k)}+\frac{1}{K_t}v_t^{(k)}$.
		\ENDFOR
		\ENDFOR
	\end{algorithmic}
\end{algorithm}
\begin{theorem}
    For online DR-submodular maximization in the i.i.d. model, if Assumption $1$ holds, the expected stochastic regret of Algorithm \ref{alg:exp4} is as follows:
    \begin{equation*}
    \E[(1-\frac{1}{e})\text{-SR}_T]\leq \sum_{t=1}^T\frac{LR^2}{2K_t}+\O(\sigma\sqrt{T}).
    \end{equation*}
    Moreover, with probability at least $1-\delta$, we have:
    \begin{equation*}
    (1-\frac{1}{e})\text{-SR}_T\leq \sum_{t=1}^T\frac{LR^2}{2K_t}+\O(\sigma\sqrt{T\ln (\frac{T^{3/2}}{\delta})}).
\end{equation*}
    Therefore, if we set $K_t=\sqrt{t}~\forall t\in[T]$, Algorithm \ref{alg:exp4} obtains an $\tilde{O}(\sqrt{T})$ stochastic regret bound, both in expectation and with high probability.
\end{theorem}
In the convex setting, $\O(\sqrt{T})$ stochastic regret bound is optimal and the same lower bound extends to the DR-submodular framework as well. Algorithm \ref{alg:exp4} manages to obtain the nearly optimal $\tilde{\O}(\sqrt{T})$ stochastic regret bound. However, in each step $t\in[T]$, the algorithm requires gradient oracle access at $\O(t^{3/2})$ points which leads to $\O(T^{5/2})$ overall gradient evaluations. Therefore, in applications where computing the gradient of utility functions is computationally expensive, Algorithm \ref{alg:exp4} has a high computational cost and may not be suitable.\\
To remedy this issue, inspired by the variance-reduction technique of \cite{xie2020efficient}, we propose Algorithm \ref{alg:exp5}. In this algorithm, for all $t\in[T]$, we estimate $\nabla f(x_{t})$ using the recursive estimator $d_t=\tilde{\nabla} f_{t}(x_{t})+(1-\rho_t)\big(d_{t-1}-\tilde{\nabla} f_t(x_{t-1})\big)$.
The regret bound of Algorithm \ref{alg:exp5} is provided below.
\begin{theorem}
    For online DR-submodular maximization in the i.i.d. model, if Assumption $1$ holds, Algorithm \ref{alg:exp5} obtains the following stochastic regret bound in expectation:
    \begin{equation*}
    \E[(\frac{1}{e})\text{-SR}_T]\leq \O(\sigma\sqrt{T}).
    \end{equation*}
    Also, the following holds with probability at least $1-\delta$:
    \begin{equation*}
    (\frac{1}{e})\text{-SR}_T\leq \O(\sigma\sqrt{T\ln (\frac{T}{\delta})}).
\end{equation*}
\end{theorem}
\begin{algorithm}[tbhp]
	\caption{}
	\begin{algorithmic}\label{alg:exp5}
		\STATE \textbf{Input:} $\{\rho_t=\frac{1}{t+1}\}_{t=1}^T$ and $T>0$ is the horizon.
		\FOR{$t=1,\dots,T$}
		\STATE Play $x_t$ and observe $f_t$.
		\IF{$t=1$}
		\STATE $d_t=\tilde{\nabla} f_t(x_t)$.
		\ELSE
		\STATE $d_t=\tilde{\nabla} f_{t}(x_{t})+(1-\rho_t)\big(d_{t-1}-\tilde{\nabla} f_t(x_{t-1})\big)$.
		\ENDIF
		\STATE $v_t=\argmax_{x\in \K}\langle x, d_t\rangle$.
		\STATE Set $x_{t+1}=x_{t}+\frac{1}{T}v_t$.
		\ENDFOR
	\end{algorithmic}
\end{algorithm}
At each round $t\in[T]$, Algorithm \ref{alg:exp4} requires computing the gradient at $\O(t^{3/2})$ points whereas Algorithm \ref{alg:exp5} requires only $2$ gradient evaluations per step. Therefore, the overall number of gradient evaluations of Algorithm \ref{alg:exp4} and Algorithm \ref{alg:exp5} are $\O(T^{5/2})$ and $\O(T)$ respectively. Moreover, both algorithms require solving just $1$ linear optimization problem over the constraint set $\K$ per round. However, despite the lower computational complexity of Algorithm \ref{alg:exp5}, this algorithm only manages to obtain an $\O(\sqrt{T})$ bound for the $(\frac{1}{e})$-stochastic regret while Algorithm \ref{alg:exp4} obtains similar bounds for the $(1-\frac{1}{e})$-regret which is the optimal approximation ratio for offline DR-submodular maximization.\\
The only prior study of online DR-submodular maximization in the i.i.d. model was done by \cite{pmlr-v80-chen18c} in which the authors proposed the OSFW algorithm with a sub-optimal $\O(T^{2/3})$ $(1-\frac{1}{e})$-stochastic regret bound in expectation. The OSFW algorithm requires only $1$ gradient evaluation per round, however, there is a mistake in the analysis of the regret bound (using the inequality $(1-\frac{1}{T})^t\leq \frac{1}{e}$ for all $t\in[T]$ which is incorrect) and consequently, the approximation ratio is in fact $\frac{1}{e}$.
\section{Numerical Examples}\label{exp}
\begin{figure}
  \centering
  \begin{tabular}[t]{ccc}
    \includegraphics[width=.3\linewidth]{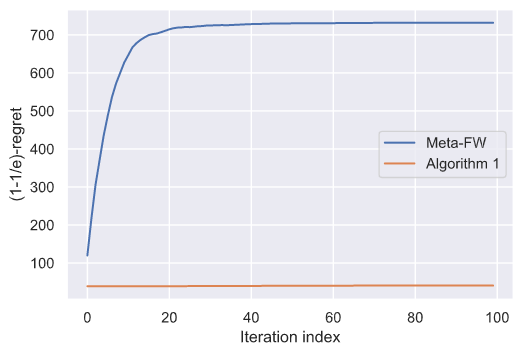} & \includegraphics[width=.33\linewidth]{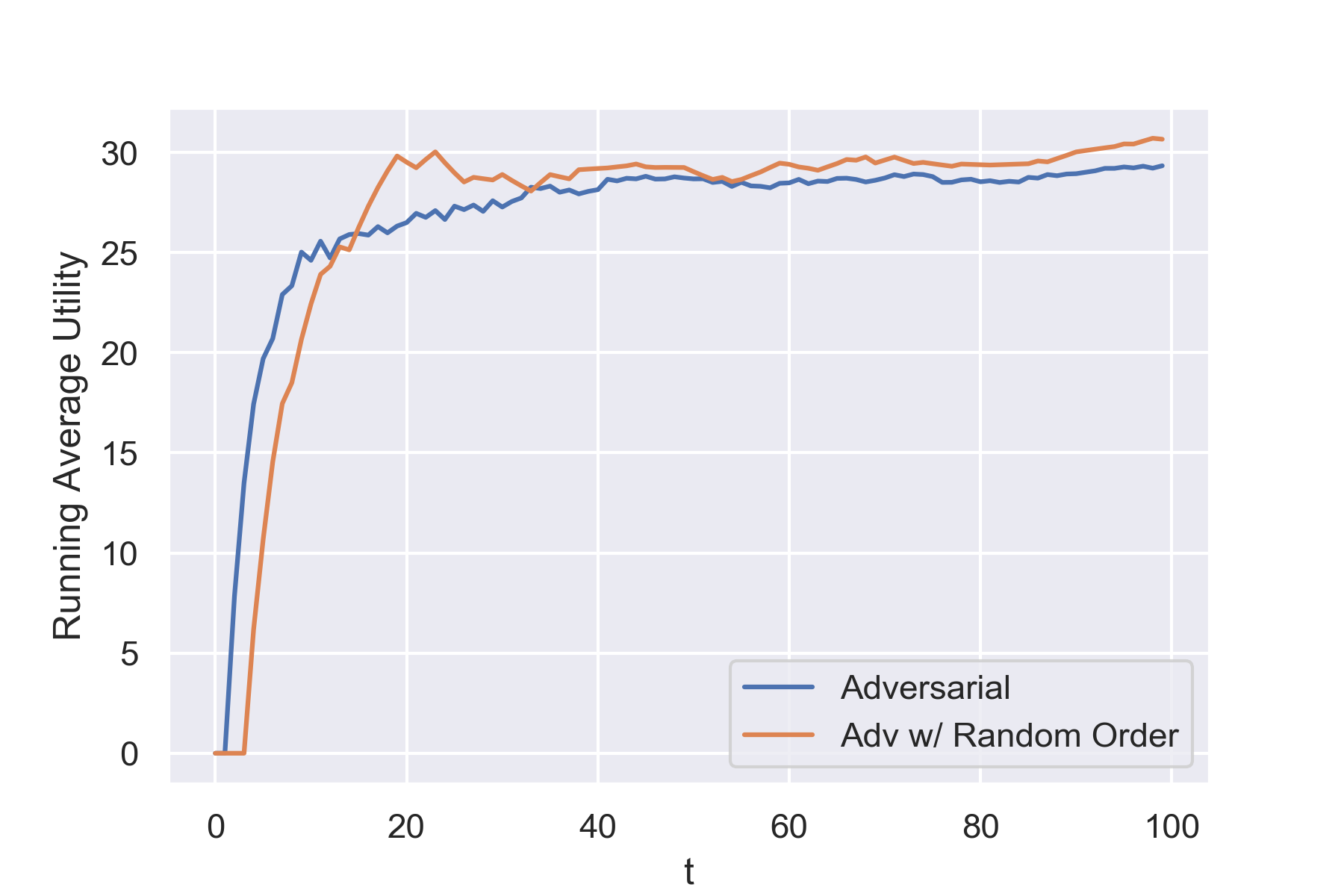} &  \includegraphics[width=.33\linewidth]{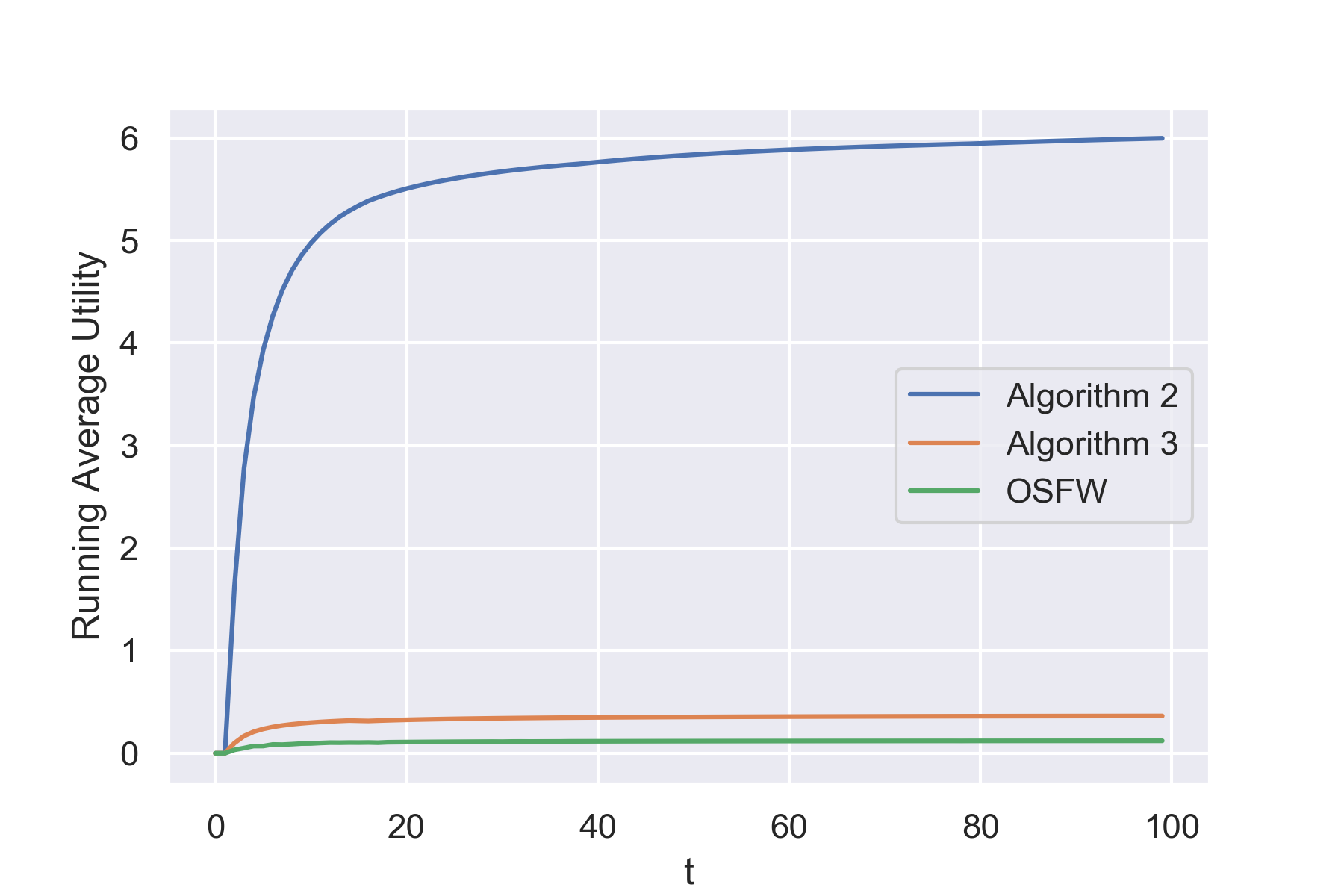}\\
    \small (a) & \small (b) & \small (c)
  \end{tabular}
  \caption{(a) Experiment 1, (b) Experiment 2, (c) Experiment 3}\label{fig}
\end{figure}
For the first experiment, we used the MovieLens dataset \cite{harper2015movielens}, containing anonymous ratings of approximately 3900 movies made by 6040 MovieLens users, and we studied a movie recommendation problem. We extracted 17 movies and 100 users with the most number of ratings. Therefore, $n=17$ and $T=100$. For all $t\in[T]$, the utility function $f_t$ is defined as $f_t(x)=5\sum_{i=1}^n \ln (1+R_i^{(t)}x_i)+\sum_{i,j:i<j}\theta_{ij}^{(t)} x_i x_j$ where $R_i^{(t)}\in[0,1]$ is the rescaled rating of $t$-th online user for the $i$-th movie, and $\theta_{ij}^{(t)}$ is uniformly distributed in the range $[-1,0]$ if the $i$-th and $j$-th movies are from the same genre. In other words, the utility function captures the diversity of the recommended movies. We chose $\K=\{x\in \R^n:1^T x\leq 4, 0\preceq x\preceq 1\}$, i.e., the algorithm has to recommend 4 movies to each arriving user. We ran Algorithm \ref{alg1} and the Meta-Frank-Wolfe algorithm of \cite{pmlr-v84-chen18f} with $K=100$, and plotted the $(1-\frac{1}{e})$-regret versus the number of iterations in Figure \ref{fig}(a). This plot verifies that the regret of Algorithm \ref{alg1} is significantly smaller than the Meta-Frank-Wolfe algorithm for strongly DR-submodular functions.\\
For the next two experiments, we set $m=2$, $n=4$ and $T=100$. We defined $\K=\{x\in \R^n:Cx\preceq b,0\preceq x\preceq 1\}$, where $C\in \R^{m\times n}$ is a matrix whose entries are uniformly distributed in the range $[0,1]$ and $b\in \R^m$ is the vector of all ones. In experiment 2, we considered quadratic utility functions of the form $f_t(x)=\frac{1}{2}x^T A^{(t)} x+(a^{(t)})^Tx$ and we chose $a^{(t)}=-(A^{(t)})^T\textbf{1}$ to ensure monotonocity of the utility functions. We chose $\{A^{(t)}\}_{t=1}^T$ as random matrices whose off-diagonal entries are uniformly distributed non-positive values in $[-10,0]$. For $t\in\{1,\dots,\frac{T}{2}\}$, we selected the diagonal entries of the matrices from the range $[-10,0]$ and for $t\in\{\frac{T}{2}+1,\dots,T\}$, these entries are uniformly distributed in the range $[0,5]$. Therefore, in this experiment, all utility functions are submodular while the average utility is $1.25$-strongly DR-submodular with respect to $\|\cdot\|_2$ and $10$-smooth with respect to $\|\cdot\|_1$ in expectation. We applied Algorithm \ref{alg1} to this problem. Then, we randomly permuted the ordering of the utility functions and applied Algorithm \ref{alg1} to blocks of size $W=5$. The result, depicted in Figure \ref{fig}(b), highlights the improved overall utility in the random order setting.\\
In the third experiment, we studied online DR-submodular maximization in the i.i.d. model. For the underlying utility function $f(x)=(\frac{1}{2}x-\textbf{1})^TAx$, we chose $A$ to be a random matrix whose entries are sampled uniformly at random from the range $[-1,0]$. Therefore, $f$ is DR-submodular and $1$-smooth. For all $t\in[T]$, we set $A^{(t)}=A+N^{(t)}$, where $N^{(t)}$ is a matrix with entries uniformly distributed in the range $[-4,4]$. Thus, $\{A^{(t)}\}_{t=1}^T$ is an unbiased estimator of $A$. Note that since $\{A^{(t)}\}_{t=1}^T$ may have positive entries, each individual utility function $\{f_t\}_{t=1}^T$ is not necessarily DR-submodular. The running average of the utility of Algorithm \ref{alg:exp4}, Algorithm \ref{alg:exp5} and the OSFW algorithm of \cite{pmlr-v80-chen18c} is depicted in Figure \ref{fig}(c). It can be observed that the regret of Algorithm \ref{alg:exp4} is smaller than the other two algorithms because the average of the utility functions observed so far is a more accurate estimation of $f$ and has a lower variance at the expense of higher computational complexity. 
\section{Conclusion and future work}\label{con}
We studied online submodular maximization in three different settings. First, we characterized the class of strongly DR-submodular functions and provided Algorithm \ref{alg1} with logarithmic regret bound in the adversarial setting. Next, we considered online submodular maximization in the random order model and showed how we can use Algorithm \ref{alg1} to obtain similar logarithmic regret bounds with high probability despite each individual utility function not being DR-submodular. Finally, we focused on online DR-submodular maximization in the i.i.d. model, and we provided two algorithms with nearly optimal $\tilde{\O}(\sqrt{T})$ regret bounds, both in expectation and with high probability. This work could be extended in a number of directions. It is interesting to see whether Algorithm \ref{alg1} could be used in a best-of-both-worlds fashion for the adversarial and random order frameworks, i.e., if we can obtain logarithmic regret bounds in the random order model by simply applying Algorithm \ref{alg1} on the individual utility functions (instead of blocks of functions). Moreover, providing an algorithm for the i.i.d. setting with optimal $\O(\sqrt{T})$ $(1-\frac{1}{e})$-stochastic regret bound and $\O(T)$ overall gradient evaluations is yet to be done.
\bibliography{refROOCO}
\newpage
\appendix
\section*{Appendix}
\section*{Missing proofs}
\subsection*{Proof of Theorem 1}
For all $t\in[T]$, using the $L$-smoothness of $f_t$, we have:
\begin{align*}
    f_t(\targ{x}{t}{k+1})&\geq f_t(\targ{x}{t}{k})+\frac{1}{K}\langle \targ{v}{t}{k},\nabla f_t(\targ{x}{t}{k})\rangle-\frac{L}{2K^2}\|\targ{v}{t}{k}\|^2.
\end{align*}
Considering that the diameter of $\K$ is $R$, taking the sum over $t\in[T]$, we obtain:
\begin{align*}
    \sum_{t=1}^T \big(f_t(\targ{x}{t}{k+1})-f_t(\targ{x}{t}{k})\big)&\geq \frac{1}{K}\sum_{t=1}^T \langle \targ{v}{t}{k},\nabla f_t(\targ{x}{t}{k})\rangle-\frac{LR^2T}{2K^2}.
\end{align*}
For all $k\in[K]$, let $\targ{v}{k}{*}=\argmax_{v\in \mathcal{K}} \big(\sum_{t=1}^T\langle v,\nabla f_t (x_t^{(k)})\rangle -\frac{\mu T}{2}\|v\|^2\big)$. If we use Follow the Leader (FTL) as $\mathcal{E}_k~\forall k\in[K]$, using Corollary 7.16 of \cite{orabona2019modern}, we can write:
\begin{align*}
    \sum_{t=1}^T\langle x^*,\nabla f_t (x_t^{(k)})\rangle -\frac{\mu T}{2}\|x^*\|^2&\overset{\text{(a)}}\leq\sum_{t=1}^T\langle \targ{v}{k}{*},\nabla f_t (x_t^{(k)})\rangle-\frac{\mu T}{2}\|\targ{v}{k}{*}\|^2\\
    &\overset{\text{(b)}}\leq\sum_{t=1}^T\big(\langle \targ{v}{t}{k},\nabla f_t (x_t^{(k)})\rangle-\frac{\mu}{2}\|\targ{v}{t}{k}\|^2\big)+\frac{(\beta+\mu R)^2}{2\mu}\ln T,
\end{align*}
where (a) uses the definition of $\targ{v}{k}{*}$ and (b) is due to the regret bound of the FTL algorithm. We have:
\begin{align*}
    \sum_{t=1}^T \big(f_t(x^*)-f_t(\targ{x}{t}{k})\big)&\overset{\text{(c)}}\leq \sum_{t=1}^T\big(f_t(x^*+ \targ{x}{t}{k})-f_t(\targ{x}{t}{k})\big)\\
    &\overset{\text{(d)}}\leq \sum_{t=1}^T \langle x^*,\nabla f_t(\targ{x}{t}{k}) \rangle -\frac{\mu T}{2}\|x^*\|^2,
\end{align*}
where (c) and (d) use the monotonocity and $\mu$-strong DR-submodularity of $\{f_t\}_{t=1}^T$.
Combining the above inequalities, we have:
\begin{align*}
    \sum_{t=1}^T \big(f_t(\targ{x}{t}{k+1})-f_t(\targ{x}{t}{k})\big)&\geq \frac{1}{K}\sum_{t=1}^T \big(f_t(x^*)-f_t(\targ{x}{t}{k})\big)+\frac{\mu}{2K}\sum_{t=1}^T \|\targ{v}{t}{k}\|^2-\frac{LR^2T}{2K^2}-\frac{(\beta+\mu R)^2}{2\mu K}\ln T\\
    &\geq \frac{1}{K}\sum_{t=1}^T \big(f_t(x^*)-f_t(\targ{x}{t}{k})\big)-\frac{LR^2T}{2K^2}-\frac{(\beta+\mu R)^2}{2\mu K}\ln T.
\end{align*}
Equivalently, we can write:
\begin{align*}
        \sum_{t=1}^T \big(f_t(\targ{x}{t}{k+1})-f_t(x^*)\big)&\geq(1-\frac{1}{K})\sum_{t=1}^T \big(f_t(\targ{x}{t}{k})-f_t(x^*)\big)-\frac{LR^2T}{2K^2}-\frac{(\beta+\mu R)^2}{2\mu K}\ln T.
\end{align*}
Applying the inequality recursively for all $k\in[K]$, we obtain:
\begin{align*}
        &\sum_{t=1}^T \big(f_t(\underbrace{\targ{x}{t}{K+1}}_{=x_t})-f_t(x^*)\big)&\geq(1-\frac{1}{K})^K\sum_{t=1}^T \big(\underbrace{f_t(\targ{x}{t}{1})}_{=0}-f_t(x^*)\big)-\frac{LR^2T}{2K}-\frac{(\beta+\mu R)^2}{2\mu}\ln T.
\end{align*}
Rearranging the terms and using the inequality $(1-\frac{1}{K})^K\leq \frac{1}{e}$, we can write:
\begin{equation*}
        \sum_{t=1}^T \big((1-\frac{1}{e})f_t(x^*)-f_t(x_t)\big)\leq \frac{LR^2T}{2K}+\frac{(\beta+\mu R)^2}{2\mu}\ln T.
\end{equation*}
Therefore, if we set $K=\O(\frac{T}{\ln T})$, we obtain $\O(\ln T)$ regret bound.
\subsection*{Proof of Theorem 3}
For non-quadratic submodular functions, the Hessian of the utility functions are not fixed and therefore, we have to ensure the strong DR-submodularity of every block of size $W$ holds at every point $x\in \K$. We provide the analysis for the class of concave functions with negative dependence below (this analysis could be easily extended to more general submodular functions). For all $t\in[T]$, let $f_t$ be of the form of concave functions with negative dependence (introduced in the paper). In this case, $\nabla^2_{ii}f_t(x)=\frac{\partial^2 \targ{h}{i}{t}}{\partial x_i^2}(x)$ for all $i\in[n]$. Note that for all $i\in[n]$, we are assuming $\frac{1}{T}\sum_{t=1}^T\targ{h}{i}{t}$ is $\mu$-strongly concave while each individual $\targ{h}{i}{t}$ may be neither convex or concave. Assume that for all $i\in[n]$ and $\tau \in \{0,\dots,\lceil\frac{T}{W}\rceil-1\}$, the second derivative of the scalar function $\frac{1}{W}\sum_{t=\tau W+1}^{(\tau+1)W}\targ{h}{i}{t}$ is $H$-Lipschitz. Also, assume that for all $x\in\K$, $0\leq x_i \leq R_i$ holds. For all $i\in[n]$, consider the discretized set of points $\mathcal{Z}_i=\{2k\gamma~\forall k\in\{0,\dots,\frac{R_i}{2\gamma}\}\}$. Note that for all $x\in\K$ and $i\in[n]$, there exists $z_i\in \mathcal{Z}_i$ such that $|x_i-z_i|\leq \gamma$. If the desired property holds for these set of points in the domain $\K$, we can use the $H$-Lipschitzness assumption to conclude
\begin{equation*}
    \frac{1}{W}\sum_{t=\tau W+1}^{(\tau+1)W}(\targ{h}{i}{t})^{''}(x_i)\leq -(\mu-\epsilon-\gamma H).
\end{equation*}
Therefore, in order to take the union bound in this setting, we set $\frac{\delta}{T\sum_{i=1}^n\lceil \frac{R_i}{2\gamma} \rceil}=4{\rm exp}\big(-\frac{W\epsilon^2}{128\theta L^2}\big)$ which is equivalent to $W_0=\frac{128\theta L^2}{\epsilon^2}\ln(\frac{4T\sum_{i=1}^n R_i}{2\gamma\delta})$.
\subsection*{Proof of Theorem 4}
For Algorithm 2, we can write:
\begin{align*}
    f(\targ{x}{t+1}{k+1})&\overset{\text{(a)}}\geq f(\targ{x}{t+1}{k})+\frac{1}{K_t}\langle \targ{v}{t}{k},\nabla f(\targ{x}{t+1}{k})\rangle -\frac{L}{2K_t^2}\|\targ{v}{t}{k}\|^2\\
    &\geq f(\targ{x}{t+1}{k})+\frac{1}{K_t}\langle \targ{v}{t}{k},\targ{d}{t}{k}\rangle+\frac{1}{K_t}\langle \targ{v}{t}{k},\nabla f(\targ{x}{t+1}{k})-\targ{d}{t}{k}\rangle -\frac{LR^2}{2K_t^2}\\
    &\overset{\text{(b)}}\geq f(\targ{x}{t+1}{k})+\frac{1}{K_t}\langle x^*,\targ{d}{t}{k}\rangle+\frac{1}{K_t}\langle \targ{v}{t}{k},\nabla f(\targ{x}{t+1}{k})-\targ{d}{t}{k}\rangle -\frac{LR^2}{2K_t^2}\\
    &= f(\targ{x}{t+1}{k})+\frac{1}{K_t}\langle x^*,\nabla f(\targ{x}{t+1}{k})\rangle+\frac{1}{K_t}\langle \targ{v}{t}{k}-x^*,\nabla f(\targ{x}{t+1}{k})-\targ{d}{t}{k}\rangle -\frac{LR^2}{2K_t^2}\\    
    &\overset{\text{(c)}}\geq f(\targ{x}{t+1}{k})+\frac{1}{K_t}\langle (x^*-\targ{x}{t+1}{k})\vee 0,\nabla f(\targ{x}{t+1}{k})\rangle+\frac{1}{K_t}\langle \targ{v}{t}{k}-x^*,\nabla f(\targ{x}{t+1}{k})-\targ{d}{t}{k}\rangle -\frac{LR^2}{2K_t^2}\\ 
    &\overset{\text{(d)}}\geq f(\targ{x}{t+1}{k})+\frac{1}{K_t}f(x^*\vee \targ{x}{t+1}{k})-\frac{1}{K_t}f(\targ{x}{t+1}{k})+\frac{1}{K_t}\langle \targ{v}{t}{k}-x^*,\nabla f(\targ{x}{t+1}{k})-\targ{d}{t}{k}\rangle -\frac{LR^2}{2K_t^2}\\
    &\overset{\text{(e)}}\geq f(\targ{x}{t+1}{k})+\frac{1}{K_t}f(x^*)-\frac{1}{K_t}f(\targ{x}{t+1}{k})+\frac{1}{K_t}\langle \targ{v}{t}{k}-x^*,\nabla f(\targ{x}{t+1}{k})-\targ{d}{t}{k}\rangle -\frac{LR^2}{2K_t^2},
\end{align*}
where (a) uses $L$-smoothness of $f$, (b) is due to the update rule of the algorithm, (c) and (e) follow from monotonocity of $f$, and (d) exploits concavity of $f$ along non-negative directions. Defining $\targ{\epsilon}{t}{k}:=\targ{d}{t}{k}-\nabla f(\targ{x}{t+1}{k})$ and rearranging the terms in the above inequality, we have:
\begin{align*}
    f(x^*)-f(\targ{x}{t+1}{k+1})&\leq (1-\frac{1}{K_t})\big(f(x^*)-f(\targ{x}{t+1}{k})\big)+\frac{R}{K_t}\|\targ{\epsilon}{t}{k}\|+\frac{LR^2}{2K_t^2}.
\end{align*}
Applying the inequality recursively for all $k\in[K_t]$ and taking expectation of both sides, we have:
\begin{align*}
    f(x^*)-\E[f(\underbrace{\targ{x}{t+1}{K_t+1}}_{=x_{t+1}})]&\leq \underbrace{(1-\frac{1}{K_t})^{K_t}}_{\leq \frac{1}{e}}\big(f(x^*)-\E[f(\underbrace{\targ{x}{t+1}{1}}_{=0})]\big)+\frac{R}{K_t}\sum_{k=1}^{K_t}\E\|\targ{\epsilon}{t}{k}\|+\frac{LR^2}{2K_t}.
\end{align*}
Rearranging the terms, taking expectation of both sides and taking the sum over $t\in \{1,\dots,T-1\}$, we obtain:
\begin{equation}\label{eq:1}
    \E[(1-\frac{1}{e})\text{-SR}_T]\leq \sum_{t=1}^{T-1}\frac{LR^2}{2K_t}+\sum_{t=1}^{T-1}\frac{R}{K_t}\sum_{k=1}^{K_t}\E\|\targ{\epsilon}{t}{k}\|+\O(1).
\end{equation}
We can bound $\E\|\targ{\epsilon}{t}{k}\|$ as follows:
\begin{align*}
    \E\|\targ{\epsilon}{t}{k}\|&=\E\sqrt{\|\targ{\epsilon}{t}{k}\|^2}\\
    &\leq \sqrt{\E\|\targ{\epsilon}{t}{k}\|^2}\\
    &= \sqrt{\E\big[\big(\frac{1}{t}\sum_{\tau=1}^t\nabla f_{\tau}(x_{t+1}^{(k)})-\nabla f(\targ{x}{t+1}{k})\big)^T\big(\frac{1}{t}\sum_{\tau=1}^t\nabla f_{\tau}(x_{t+1}^{(k)})-\nabla f(\targ{x}{t+1}{k})\big)\big]}\\
    &= \sqrt{\frac{1}{t^2}\sum_{\tau=1}^t\E\|\nabla f_{\tau}(x_{t+1}^{(k)})-\nabla f(\targ{x}{t+1}{k})\|^2}\\
    &\leq \sqrt{\frac{1}{t^2}\sum_{\tau=1}^t \sigma^2}\\
    &=\frac{\sigma}{\sqrt{t}},
\end{align*}
where the first inequality is due to Jensen's inequality. Therefore, if we set $K_t=\sqrt{t}~\forall t\in[T]$, the expected regret bound of the algorithm is $\O((LR^2+R\sigma)\sqrt{T})$.\\
We can also obtain high probability regret bound for the algorithm. Considering that $f$ is $\beta$-Lipschitz, we have $\|\nabla f_t(\cdot)\|\leq \|\nabla f(\cdot)\|+\|\nabla f_t(\cdot)-\nabla f(\cdot)\|\leq \beta+\sigma$. Therefore, we can use Corollary 7 of \cite{jin2019short} to conclude that with probability at least $1-\frac{\delta}{T^{3/2}}$, the following holds:
\begin{equation*}
    \|\targ{\epsilon}{t}{k}\|\leq \O\big((\sigma+\beta)\sqrt{\frac{\ln(2T^{3/2}/\delta)}{t}}\big).
\end{equation*}
Taking the union bound, we can conclude that $(1-\frac{1}{e})\text{-SR}_T\leq \sum_{t=1}^{T-1}\frac{LR^2}{2K_t}+\O(\sigma\sqrt{T\ln (T^{3/2}/\delta)})$ holds with probability at least $1-\delta$. Thus, if we set $K_t=\sqrt{t}~\forall t\in[T]$, we obtain an $\tilde{\O}(\sqrt{T})$ regret bound, both in expectation and with high probability.
\subsection*{Proof of Theorem 5}
Using the $L$-smoothness of $f$ and the update rule of Algorithm 3, we have:
\begin{align*}
    f(x_{t+1})&\overset{\text{(a)}}\geq f(x_t)+\frac{1}{T}\langle v_t,\nabla f(x_t)\rangle -\frac{L}{2T^2}\|v_t\|^2\\
    &\geq f(x_t)+\frac{1}{T}\langle v_t,d_t\rangle + \frac{1}{T}\langle v_t,\nabla f(x_t)-d_t\rangle-\frac{LR^2}{2T^2}\\
    &\overset{\text{(b)}}\geq f(x_t)+\frac{1}{T}\langle x^*,d_t\rangle + \frac{1}{T}\langle v_t,\nabla f(x_t)-d_t\rangle-\frac{LR^2}{2T^2}\\
    &=f(x_t)+\frac{1}{T}\langle x^*,\nabla f(x_t)\rangle + \frac{1}{T}\langle v_t-x^*,\nabla f(x_t)-d_t\rangle-\frac{LR^2}{2T^2}\\
    &\overset{\text{(c)}}\geq f(x_t)+\frac{1}{T}\langle (x^*-x_t)\vee 0,\nabla f(x_t)\rangle+\frac{1}{T}\langle v_t-x^*,\nabla f(x_t)-d_t\rangle -\frac{LR^2}{2T^2}\\ 
    &\overset{\text{(d)}}\geq f(x_t)+\frac{1}{T}f(x^*\vee x_t)-\frac{1}{T}f(x_t)+\frac{1}{T}\langle v_t-x^*,\nabla f(x_t)-d_t\rangle -\frac{LR^2}{2T^2}\\
    &\overset{\text{(e)}}\geq f(x_t)+\frac{1}{T}f(x^*)-\frac{1}{T}f(x_t)+\frac{1}{T}\langle v_t-x^*,\nabla f(x_t)-d_t\rangle -\frac{LR^2}{2T^2},
\end{align*}
where (a) uses $L$-smoothness of $f$, (b) is due to the update rule of the algorithm, (c) and (e) follow from monotonocity of $f$, and (d) exploits concavity of $f$ along non-negative directions. Defining $\epsilon_t:=d_t-\nabla f(x_t)$ and rearranging the terms in the above inequality, we have:
\begin{align*}
    f(x^*)-f(x_{t+1})&\leq (1-\frac{1}{T})\big(f(x^*)-f(x_t)\big)+\frac{R}{T}\|\targ{\epsilon}{t}{k}\|+\frac{LR^2}{2T^2}.
\end{align*}
Applying the above inequality recursively, we have:
\begin{align}
    f(x^*)-f(x_{t+1})&\leq \underbrace{(1-\frac{1}{T})^{t}}_{\leq e^{-t/T}}\big(f(x^*)-\underbrace{f(x_1)}_{=0}\big)+\frac{R}{T}\sum_{s=1}^{t}\|\epsilon_s\|+\frac{LR^2}{2T}.\label{eq:2}
\end{align}
Using inequality \ref{eq:2} and $\sum_{t=1}^{T-1}e^{-t/T}\leq T(e^{-1/T}-\frac{1}{e})\leq T(1-\frac{1}{e})$, we have $(\frac{1}{e})\text{-SR}_T\leq \frac{LR^2}{2}+\frac{R}{T}\sum_{t=1}^{T-1}\sum_{s=1}^t\|\epsilon_{s}\|$. We can write:
\begin{align*}
    \epsilon_{t}&=d_{t}-\nabla f(x_t)\\
    &= (1-\rho_t)\epsilon_{t-1}+\rho_t\big(\nabla f_t(x_{t})-\nabla f(x_{t})\big)\\
    &+(1-\rho_t)\big(\nabla f_t(x_{t})-\nabla f_t(x_{t-1})-(\nabla f(x_{t})-\nabla f(x_{t-1}))\big).
\end{align*}
Applying the above equality recursively, we obtain:
\begin{align*}
    \epsilon_t&=\prod_{s=2}^t(1-\rho_s)\epsilon_1+\sum_{\tau=2}^t\prod_{s=\tau}^t(1-\rho_s)\big(\nabla f_{\tau}(x_{\tau})-\nabla f_{\tau}(x_{\tau-1})-(\nabla f(x_{\tau})-\nabla f(x_{\tau-1}))\big)\\
    &+\sum_{\tau=2}^t\rho_{\tau}\prod_{s=\tau+1}^t(1-\rho_s)\big(\nabla f_{\tau}(x_{\tau})-\nabla f(x_{\tau})\big).
\end{align*}
Let $\epsilon_t=\sum_{\tau=1}^t\zeta_{t,\tau}$, where $\zeta_{t,1}=\prod_{s=2}^t(1-\rho_s)\epsilon_1$ and $\zeta_{t,\tau}=\prod_{s=\tau}^t(1-\rho_s)\big(\nabla f_{\tau}(x_{\tau})-\nabla f_{\tau}(x_{\tau-1})-(\nabla f(x_{\tau})-\nabla f(x_{\tau-1}))\big)+\rho_{\tau}\prod_{s=\tau+1}^t(1-\rho_s)\big(\nabla f_{\tau}(x_{\tau})-\nabla f(x_{\tau})\big)$ for $\tau>1$. Let $\mathcal{F}_{\tau-1}$ be the $\sigma$-field generated by $\{f_s\}_{s=1}^{\tau-1}$. Clearly, $\E[\zeta_{t,1}]=0$. Also, for $\tau>1$, we have:
\begin{align*}
    \E[\zeta_{t,\tau}|\mathcal{F}_{\tau-1}]&=\rho_{\tau}\prod_{s=\tau+1}^t(1-\rho_s)\big(\nabla f(x_{\tau})-\nabla f(x_{\tau})\big)\\
    &+\prod_{s=\tau}^t(1-\rho_s)\big(\nabla f(x_{\tau})-\nabla f(x_{\tau-1})-(\nabla f(x_{\tau})-\nabla f(x_{\tau-1}))\big)\\
    &=0.
\end{align*}
Therefore, for all $t\in[T]$, $\{\zeta_{t,\tau}\}_{\tau=1}^t$ is a martingale difference sequence.\\
For any $\tau\in[t]$, we can write:
\begin{equation*}
    \prod_{s=\tau}^t(1-\rho_s)=\prod_{s=\tau}^t(1-\frac{1}{s+1})=\prod_{s=\tau}^t\frac{s}{s+1}=\frac{\tau}{t+1}.
\end{equation*}
Thus, we have $\|\zeta_{t,1}\|=\frac{2}{t+1}\|\nabla f_1(x_1)-\nabla f(x_1)\|\leq \frac{2\sigma}{t+1}$. For $\tau>1$, $\|\zeta_{t,\tau}\|$ could be bounded as follows:
\begin{align*}
    \|\zeta_{t,\tau}\|&\leq \prod_{s=\tau}^t(1-\rho_s)\big(\|\nabla f_{\tau}(x_{\tau})-\nabla f_{\tau}(x_{\tau-1})\|+\|\nabla f(x_{\tau})-\nabla f(x_{\tau-1})\|\big)\\
    &+\rho_{\tau}\prod_{s=\tau+1}^t(1-\rho_s)\|\nabla f_{\tau}(x_{\tau})-\nabla f(x_{\tau})\|\\
    &\leq \frac{2L\tau}{t+1}\|\underbrace{x_{\tau}-x_{\tau-1}}_{\frac{1}{T}v_t}\|+\frac{\sigma}{t+1}\\
    &\leq \frac{2LR\tau/T+\sigma}{t+1}\\
    &\leq \frac{2LR+\sigma}{t+1}.
\end{align*}
Using the concentration inequality for vector-valued martingales, we have:
\begin{align*}
    \mathbb{P}\big(\|\epsilon_t\|\geq \lambda_t\big)&\leq 4{\rm exp}\big(-\frac{\lambda_t^2}{(\frac{2\sigma}{t+1})^2+(t-1)(\frac{2LR+\sigma}{t+1})^2}\big)\\
    &\leq  4{\rm exp}\big(-\frac{\lambda_t^2(t+1)}{(2LR+2\sigma)^2}\big).
\end{align*}
Therefore, if we set $\lambda_t=\frac{(2LR+2\sigma)\sqrt{\ln (4T/\delta)}}{\sqrt{t+1}}$, for all $t\in[T]$, with probability at least $1-\delta$, the following holds:
\begin{equation*}
    \|\epsilon_t\|\leq \frac{(2LR+2\sigma)\sqrt{\ln (4T/\delta)}}{\sqrt{t+1}}.
\end{equation*}
Thus, with probability at least $1-\delta$, the regret bound is $\O(\sigma \sqrt{T\ln(T/\delta)})$.\\
We can also obtain the expected regret bound of the algorithm. We have:
\begin{align*}
    \E\|\epsilon_t\|&=\int_{\lambda=0}^{\infty}\mathbb{P}(\|\epsilon_t\|\geq \lambda)d\lambda\\
    &\leq \int_{\lambda=0}^{\infty}4{\rm exp}\big(-\frac{\lambda^2(t+1)}{(2LR+2\sigma)^2}\big)d\lambda\\
    &=\int_{x=0}^{\infty}4{\rm exp}(-x^2)\frac{2LR+2\sigma}{\sqrt{t+1}}dx\\
    &=\frac{4\sqrt{\pi}(LR+\sigma)}{\sqrt{t+1}}.
\end{align*}
Therefore, using inequality \ref{eq:2}, the expected regret bound is $\O(\sigma\sqrt{T})$.
\end{document}